\documentclass{article}

\usepackage[utf8]{inputenc} 
\usepackage[margin=1.008in]{geometry}
\usepackage[T1]{fontenc}    
\usepackage{hyperref}       
\usepackage{url}            
\usepackage{booktabs}       
\usepackage{amsfonts}       
\usepackage{nicefrac}       
\usepackage{microtype}      
\usepackage{xcolor}         

\usepackage{tikz}

\title{Online Learning in Contextual \\Second-Price Pay-Per-Click Auctions}

\usepackage[utf8]{inputenc} 
\usepackage[T1]{fontenc}    
\usepackage{hyperref}       
\usepackage{url}            
\usepackage{booktabs}       
\usepackage{amsfonts}       
\usepackage{nicefrac}       
\usepackage{microtype}      
\usepackage{xcolor}         
\usepackage{thmtools}
\usepackage{resizegather}
\usepackage{adjustbox}

\newif\ifspacehack
\usepackage{natbib}
\hypersetup{
    colorlinks = blue,
    breaklinks,
    linkcolor = blue,
    citecolor = blue,
    urlcolor  = blue,
}
\usepackage{url} 
\usepackage{graphicx}
\usepackage{mathtools}
\usepackage{footnote}
\usepackage{float}
\usepackage{xspace}
\usepackage{multirow}
\usepackage{xcolor}
\usepackage{wrapfig}
\usepackage{framed}
\usepackage{bbm}
\usepackage{footnote}
\usepackage{nicefrac}
\usepackage{makecell}
\usepackage[algo2e,noend]{algorithm2e} 
\usepackage{algorithm}
\usepackage{amssymb}
\usepackage{amsthm}
\usepackage{bm}
\makesavenoteenv{tabular}
\makesavenoteenv{table}

\renewcommand{\tilde}{\widetilde}

\newtheorem{proposition}{Proposition}[section]
\newtheorem{lemma}{Lemma}[section]

\newcommand{\expfour}{\ensuremath{\mathsf{Exp4}}\xspace}

\def \R {\mathbb{R}}

\newcommand{\LS}{\mathrm{LS}}
\newcommand{\FG}{\mathrm{FG}}

\newcommand{\calA}{{\mathcal{A}}}
\newcommand{\smax}{{\mathrm{smax}}}

\newcommand{\calX}{{\mathcal{X}}}

\newcommand{\calF}{{\mathcal{F}}}

\newcommand{\calE}{{\mathcal{E}}}

\newcommand{\Reg}{\text{\rm Reg}}

\newcommand{\one}{\boldsymbol{1}}

                                                                                                                                                                                                                                                                                                                                                                                                                                                                                                                                                                                                                                                                                                                                                                                                                                                                                                                                                                                                                                                                          \newcommand{\sigmoid}{\ensuremath{\mathsf{Sigmoid}}\xspace}

\DeclareMathOperator*{\argmax}{argmax}
\DeclareMathOperator*{\argsmax}{argsmax}

\newcommand{\E}{\field{E}}

\newcommand{\inner}[1]{ \left\langle {#1} \right\rangle }

\newcommand{\wh}{\widehat}
\newcommand{\wt}{\widetilde}

\newcommand{\ellhat}{\wh{\ell}}

\newcommand{\AlgSq}{\ensuremath{\mathsf{AlgSq}}\xspace}
\newcommand{\ips}{\ensuremath{\mathsf{(IPS)}}\xspace}
\newcommand{\optsq}{\ensuremath{\mathsf{(OptSq)}}\xspace}
\newcommand{\sq}{\ensuremath{\mathsf{(Sq)}}\xspace}
\newcommand{\dec}{\ensuremath{\mathsf{dec}_\gamma}\xspace}

\newtheorem{assumption}{Assumption}

\newcommand{\order}{\ensuremath{\mathcal{O}}}
\newcommand{\otil}{\ensuremath{\tilde{\mathcal{O}}}}

\usepackage{lipsum,booktabs}
\usepackage{amsmath,mathrsfs,amssymb,amsfonts,bm,enumitem}
\usepackage{rotating}
\usepackage{pdflscape}
\usepackage{hyperref,url}
\hypersetup{
    colorlinks,
    breaklinks,
    linkcolor = blue,
    citecolor = blue,
    urlcolor  = blue,
}
\allowdisplaybreaks
\usepackage{appendix}
\usepackage{multirow,makecell,tabularx}

\usepackage{algorithmic,algorithm}

\renewcommand{\tilde}{\widetilde}

\def \E {\mathbb{E}}

\def \R {\mathbb{R}}

\def \Pr {\mathsf{Pr}}

\newcommand{\RegSq}{\ensuremath{\mathrm{\mathbf{Reg}}_{\mathsf{Sq}}}\xspace}

\usepackage{mathtools}

\usepackage{amsthm}

\theoremstyle{definition}

\usepackage{graphicx,subfigure,color} 

\definecolor{wine_red}{RGB}{228,48,64}
\definecolor{DSgray}{cmyk}{0,1,0,0}

\usepackage{prettyref}
\newcommand{\pref}[1]{\prettyref{#1}}

\newcommand{\savehyperref}[2]{\texorpdfstring{\hyperref[#1]{#2}}{#2}}
\newrefformat{eq}{\savehyperref{#1}{Eq. \textup{(\ref*{#1})}}}
\newrefformat{eqn}{\savehyperref{#1}{Eq.~(\ref*{#1})}}
\newrefformat{lem}{\savehyperref{#1}{Lemma~\ref*{#1}}}
\newrefformat{def}{\savehyperref{#1}{Definition~\ref*{#1}}}
\newrefformat{line}{\savehyperref{#1}{Line~\ref*{#1}}}
\newrefformat{thm}{\savehyperref{#1}{Theorem~\ref*{#1}}}
\newrefformat{corr}{\savehyperref{#1}{Corollary~\ref*{#1}}}
\newrefformat{cor}{\savehyperref{#1}{Corollary~\ref*{#1}}}
\newrefformat{sec}{\savehyperref{#1}{Section~\ref*{#1}}}
\newrefformat{app}{\savehyperref{#1}{Appendix~\ref*{#1}}}
\newrefformat{assum}{\savehyperref{#1}{Assumption~\ref*{#1}}}
\newrefformat{asm}{\savehyperref{#1}{Assumption~\ref*{#1}}}
\newrefformat{ex}{\savehyperref{#1}{Example~\ref*{#1}}}
\newrefformat{fig}{\savehyperref{#1}{Figure~\ref*{#1}}}
\newrefformat{alg}{\savehyperref{#1}{Algorithm~\ref*{#1}}}
\newrefformat{rem}{\savehyperref{#1}{Remark~\ref*{#1}}}
\newrefformat{conj}{\savehyperref{#1}{Conjecture~\ref*{#1}}}
\newrefformat{prop}{\savehyperref{#1}{Proposition~\ref*{#1}}}
\newrefformat{proto}{\savehyperref{#1}{Protocol~\ref*{#1}}}
\newrefformat{prob}{\savehyperref{#1}{Problem~\ref*{#1}}}
\newrefformat{claim}{\savehyperref{#1}{Claim~\ref*{#1}}}
\newrefformat{que}{\savehyperref{#1}{Question~\ref*{#1}}}
\newrefformat{op}{\savehyperref{#1}{Open Problem~\ref*{#1}}}
\newrefformat{fn}{\savehyperref{#1}{Footnote~\ref*{#1}}}

\def \epsilon {\varepsilon}

\usepackage{color-edits}
\addauthor{hl}{red}     
\addauthor{mz}{blue}  
\author{%
  Mengxiao Zhang \\
  University of Southern California \\  \texttt{mengxiao.zhang@usc.edu} \\ 
  \and
  Haipeng Luo \\
  University of Southern California \\  \texttt{haipengl@usc.edu}\\
}

\begin{document}
\maketitle

\begin{abstract}
We study online learning in contextual pay-per-click auctions where at each of the $T$ rounds, the learner receives some context along with a set of ads and needs to make an estimate on their click-through rate (CTR) in order to run a second-price pay-per-click auction.
The learner's goal is to minimize her regret, defined as the gap between her total revenue and that of an oracle strategy that always makes perfect CTR predictions.
We first show that $\sqrt{T}$-regret is obtainable via a computationally inefficient algorithm and that it is unavoidable since our algorithm is no easier than the classical multi-armed bandit problem. A by-product of our results is a $\sqrt{T}$-regret bound for the simpler non-contextual setting, improving upon a recent work of~\citet{ICML'23:auction} by removing the inverse CTR dependency that could be arbitrarily large. Then, borrowing ideas from recent advances on efficient contextual bandit algorithms, we develop two practically efficient contextual auction algorithms: the first one uses the exponential weight scheme with optimistic square errors and maintains the same $\sqrt{T}$-regret bound, while the second one reduces the problem to online regression via a simple epsilon-greedy strategy, albeit with a worse regret bound.
Finally, we conduct experiments on a synthetic dataset to showcase the effectiveness and superior performance of our algorithms.
\end{abstract}

\section{Introduction}\label{sec: intro}

The rapid growth of internet-based advertising has led to increasing reliance on online auctions to efficiently allocate advertisement slots. Pay-per-click (PPC) auctions, in particular, have become a prevalent mechanism in the digital advertising landscape, where advertisers are charged based on the number of clicks on their ads. In these auctions, a platform's primary goal is to deliver the relevant experience while maximizing value to the advertiser and the publisher.

Existing literature on online PPC auctions mostly focuses on the non-contextual setup~\citep{devanur2009price,buccapatnam2014stochastic,babaioff2015truthful,ICML'23:auction}, where the same set of ads repeatedly participates in an auction, each with a click-through rate (CTR) fixed over time.
In reality, however, the set of participating ads and their CTR vary in each auction based on the user query, user preferences/history, ad relevance, and other contextual information.
To tackle such practical scenarios, in this work, we consider online \textit{contextual} PPC auctions with unknown CTRs and study how the auction platform can leverage the contextual information to make a revenue close to that of an oracle strategy that runs a second-price auction with perfect knowledge of the CTRs.

More concretely, we formulate this problem as an online learning problem over $T$ rounds. At each round, the auction platform (learner) first observes some context and a set of participating ads, and then makes an estimate for the CTR of each ad without seeing their bid.
Afterwards, a second-price PPC auction, a truthful and widely used mechanism~\citep{aggarwal2006truthful}, is run: each ad is assigned a score equal to the product of its bid and its estimated CTR, and the ad with the highest score wins the auction with the payment-per-click equal to the \textit{critical price} (that is, the lowest price that still guarantees a win).
The learner's goal is to decide the estimated CTRs in a way so that the total revenue is close to what one would have received if the CTR estimations were always perfect --- we call the gap between them the \textit{regret} of the learner.
To make sublinear regret possible, we make a standard realizability assumption that the learner is given access to a CTR predictor class that contains a perfect and unknown predictor, but we do not make any assumption on how the contexts and the bids are generated --- they can even be maliciously decided by an adversary.

To our knowledge, our work is the first to consider online learning for such contextual PPC auctions.
However, similar to its non-contextual version, the problem has deep connections with the heavily studied (contextual) multi-armed bandit problem~\citep{lai1985asymptotically, auer2002nonstochastic}.
In particular, because we only observe feedback on the winner, balancing exploration and exploitation, the infamous dilemma originated from multi-armed bandits, is also the key challenge of our problem.
What makes our problem even more difficult is that we cannot explore/exploit whichever ad/arm we want but instead have to do so implicitly via setting the estimated CTRs, which themselves by definition also affect the reward of each ad/arm.
Despite these difficulties, by extending ideas from contextual bandits and making careful adjustment tailored to our problem structure, we obtain a series of positive results both theoretically and empirically. Specifically, our contributions are:

\begin{enumerate}
[leftmargin=*]
\item As the first step, in \pref{sec:EXP4}, we provide a characterization of the optimal regret of our problem via a computationally inefficient algorithm and a simple lower-bound argument showing that our problem is no easier than multi-armed bandits.
Our algorithm is based on the well-known exponential weight strategy, with an inverse propensity score (IPS) weighted loss estimator that is similar to the \expfour algorithm~\citep{auer2002nonstochastic} for contextual bandits.
For a finite predictor class $\calF$, our algorithm achieves $\order(\sqrt{NT\log|\calF|})$ near-optimal regret with $N$ being the number of ads. Notably, our result immediately implies $\otil(N\sqrt{T})$ regret for the non-contextual setting, improving upon the $\otil(\sum_{i=1}^N\frac{1}{\rho_i}\sqrt{T})$ regret of a recent work by~\citet{ICML'23:auction} where $\rho_i\in [0,1]$ is the CTR of ad $i$ (whose inverse could be arbitrarily large). 

\item To address the computational inefficiency, we then develop two practically efficient algorithms, taking inspiration from recent developments in designing efficient contextual bandit algorithms.
Our first approach (\pref{sec: TS}) replaces the IPS estimator with an \emph{optimistic square error estimator} that is efficiently computable and shares similar ideas with the Feel-Good Thompson Sampling algorithm of~\citet{zhang2022feel}.
The resulting algorithm not only still enjoys a $\sqrt{T}$-regret bound, but also admits an efficient (approximate) implementation by applying stochastic gradient Langevin dynamics (SGLD)~\citep{welling2011bayesian}.

\item Our second approach (\pref{sec: squarecb}) follows another trend of recent studies that reduce contextual bandits to an easier regression problem where efficient algorithms already exist~\citep{foster2020beyond,foster2021efficient,foster2021statistical}.
We adopt the general framework of~\citep{foster2021statistical} in attempt to find the optimal reduction from our contextual auction problem to online regression.
We provide a simple solution based on an epsilon-greedy strategy that is efficiently implementable and achieves $\order(T^{\frac{2}{3}}(N\RegSq)^{\frac{1}{3}})$ regret with $\RegSq$ being the regret of the regression problem.
While this method leads to a worse regret bound, we conjecture that it might be the best one can do using this reduction approach due to its distinct property: unlike the last two algorithms, this approach does not use the bid information from previous rounds to decide the next CTR estimates (which sometimes might be desirable in practice).

\item In \pref{sec:experiment}, we also test our two efficient algorithms on a synthetic dataset, demonstrating their superior performance against several baseline algorithms.
\end{enumerate}

\paragraph{Related works.} 

One line of closely related work is the study on the non-contextual counterpart, such as~\citep{devanur2009price,babaioff2014characterizing,babaioff2015truthful,ICML'23:auction}, all of which consider designing a globally truthful no-regret mechanism so that bidders are incentivized to bid their true valuation throughout all $T$ rounds.
\citep{devanur2009price, babaioff2014characterizing} achieves so via an explore-then-commit strategy with $\order(N^{\frac{1}{3}}T^{\frac{2}{3}})$ regret, which is shown to be optimal for globally truthful mechanisms.
\citep{babaioff2015truthful} further considers the setting where each bidder's bid is stochastic and designs a randomized auction that enjoys $\order(\sqrt{T})$ regret and is globally truthful only in expectation. Further extensions include multi-slot mechanism design~\citep{gatti2012truthful}, valuations unknown to the bidders~\citep{kandasamy2023vcg}, and others. 

A recent work by \citet{ICML'23:auction} considers the myopic bidder setting with adversarial bids and designs a UCB-based algorithm which leads to a per-round truthful auction and achieves $\otil(\sum_{i=1}^N\frac{1}{\rho_i}\sqrt{T})$ regret. They also consider the fixed valuation setting and design a globally truthful auction with $-\Omega(T)$ regret when there exists a time-independent constant gap between the winner and the runner up. More recently, \citep{xu2023robustness} generalizes this non-contextual setting to the stochastic context setting and derives an $\epsilon$-greedy-based algorithm achieving $\otil(N^{\frac{4}{3}}T^{\frac{2}{3}}+\frac{1}{\alpha^2}T^{\frac{1}{3}}N^{\frac{2}{3}})$ regret under $\alpha$-rational bidders.
We note that in our setting, similar to~\citep{ICML'23:auction}, the learner is allowed to adjust the auction on the fly based on previous observation (but not the bids for the current round), which makes it truthful per round but not necessarily globally.
However, since we allow \emph{adversarial} contexts that might not be manipulatable by the bidders, we do not find global truthfulness a meaningful requirement for our setting (see \pref{sec: pre} for more discussion).

Another line of work on online learning in auction considers designing auto-bidding algorithms under different types of auction mechanisms (such as first-price auction~\citep{wang2023learning,han2020learning}, second-price auction~\citep{balseiro2019learning,balseiro2023best}, and core auction~\citep{gaitonde2022budget}) or different types of resource and return-on-investment constraints~\citep{balseiro2019learning,balseiro2021robust,lucier2023autobidders}.
Since we allow adversarial bids, our results hold whether the platform is facing these auto-bidding algorithms or not.

As mentioned, our problem bares some similarity with the heavily-studied contextual bandit problem but is generally more difficult due to the fact that both the way to select an arm and its reward are determined through the estimated CTRs.
\citep{auer2002nonstochastic} proposes the first contextual bandit algorithm~\expfour which achieves the optimal regret but is computationally inefficient. 
Given the vast number of real-life applications of contextual bandits, starting from~\citep{langford2007epoch}, many studies focus on developing practically efficient algorithm under reasonable computational assumptions, such as access to classification/regression oracles~\citep{dudik2011efficient, agarwal2014taming, foster2018practical, foster2020beyond,xu2020upper,foster2021efficient,simchi2022bypassing,zhu2022contextual}, or ability to sample from a certain distribution using Markov chain Monte Carlo methods~\citep{zhang2022feel}.
As mentioned, we follow and extend the ideas of these two approaches to design efficient  contextual auction algorithms.

\section{Notations and Problem Setup}\label{sec: pre}

\paragraph{General notations.} 
For a positive integer $m$, we use $[m]$ to denote the set $\{1,2,\dots,m\}$, $\Delta_m$ to denote the $(m-1)$-dimensional simplex, and $\R^{m}_{\geq 0}$ to denote the set of $m$-dimensional vectors with non-negative entries.
For a vector $v\in\R^m$, we use $\max_{i\in[m]}v_i$ and $\smax_{i\in[m]}v_i$ to denote the largest and the second largest entry of $v$ respectively, and $\argmax_{i\in[m]}v_i$ and $\argsmax_{i\in[m]}v_i$ to denote their index.\footnote{In this definition, we break ties by an arbitrary fixed deterministic rule. Note that when there is a tie, it is possible that the ``second largest'' entry in fact has the largest value.}
We also use $\one$ to denote the all-one vector and $e_i$ to denote the $i$-th standard basis vector (both with an appropriate dimension depending on the context).

\paragraph{Problem setup.}
The formal setup of the contextual second-price pay-per-click auctions we consider is as follows.
An ad auction platform (called the learner) sequentially interacts with some bidders/advertisers for $T$ rounds.
At each round $t\in[T]$:
\begin{enumerate}[leftmargin=*]
\item The learner first observes a context $x_t$ from some context space $\calX$ and a set of $N_t$ bidders participating in the current campaign with their ads.
Here, the context $x_t$ encodes any available information about the current campaign, such as the user query (in the case of a search engine) and features of the participating ads.
We denote the maximum number of bidders by $N = \max_{t} N_t$.

\item What the learner needs to decide is an estimated CTR vector: $\wt{\rho}_{t}\in [0,1]^{N_t}$, with each entry $\wt{\rho}_{t,i}$ being an estimation of the true and unknown CTR $\rho_{t,i} \in [0,1]$ of ad $i$ under context $x_t$.

\item Simultaneously, each bidder $i$ decides their own bid $b_{t,i} \in [0,1]$ (without knowing $\wt{\rho}_t$ or $\rho_{t}$).

\item A second-price pay-per-click auction~\citep{aggarwal2006truthful} is then run: the winner of this campaign is $i_t=\argmax_{i\in[N_t]}b_{t,i}\wt{\rho}_{t,i}$, that is, the bidder with the highest estimated expected cost per impression;
the payment per click of the winner is $d_t=\frac{b_{t,j_t}\wt{\rho}_{t,j_t}}{\wt{\rho}_{t,i_t}}$ where $j_t = \argsmax_{i\in[N_t]} b_{t,i}\wt{\rho}_{t,i}$ is the runner-up (note $d_t \leq b_{t,i_t}$ by definition, so the winner never pays more than their bid);
the winner's ad is then displayed, and is clicked with probability $\rho_{t,i_t}$.

\item The feedback of the learner includes all the bids $b_{t,1}, \ldots, b_{t,N_t}$ and a binary variable $c_t \in \{0,1\}$, which is $1$ if the displayed ad is clicked and $0$ otherwise (by definition, $c_t$ is a Bernoulli random variable with mean $\rho_{t,i_t}$).
The payment that the learner receives in the end is thus $c_t d_t$.
\end{enumerate}

The goal of the learner is to minimize her regret, which measures how much more she would have received if she had perfect knowledge of the true CTRs and set $\wt{\rho}_t = \rho_t$ all the time.
More concretely, since in this imaginary situation with $\wt{\rho}_t = \rho_t$, the expected payment received at round $t$ is $\E[c_t d_t] = \rho_{t,i_t}\cdot\frac{b_{t,j_t}\rho_{t,j_t}}{\rho_{t,i_t}} = b_{t,j_t}\rho_{t,j_t} =  \smax_{i\in [N_t]} b_{t,i} \rho_{t,i}$,
the regret is formally defined as
\begin{align}\label{eqn:regret}
    \Reg \triangleq \sum_{t=1}^T\smax_{i\in[N_t]}\; b_{t,i}\rho_{t,i} - \sum_{t=1}^Tc_td_t.
\end{align}
Achieving sublinear (in $T$) regret thus implies that the learner is on average performing almost as well as the ideal benchmark even though she does not know the true CTRs.
This is a strong requirement that intuitively is  possible only if there is some connection between the context $x_t$ and the true CTR vector $\rho_t$ so that over time the learner can gradually improve her estimation $\wt{\rho}_t$ based on prior observations.
To this end, we make the following realizable assumption that is analogous to those usually made in the contextual bandit literature (see for example~\citep{foster2020beyond,foster2021efficient,foster2021statistical,zhu2022contextual,simchi2022bypassing}).

\begin{assumption}[Realizability]\label{asm:realizability}
A function class $\calF = \{f: \calX \times [N] \rightarrow [0,1]\}$, given to the learner, contains a perfect (but unknown) CTR predictor $f^*$ such that $f^*(x_t, i) = \rho_{t,i}$ for all $t$ and $i$.
\end{assumption}

Our goal is to derive algorithms with a regret bound that is sublinear in $T$ and polynomial in $N$ and some common complexity measure of $\calF$, such as $\ln |\calF|$ for the case of a finite class or the number of parameters for the case of a parametric class.
In practice, such a function class $\calF$ can be any common machine learning model (for example, a linear class or a class of neural nets), as long as it is believed to be complex enough to predict the true CTRs reasonably well for the specific problem on hand.

We emphasize that we do not make any assumptions on how the contexts and the bids are generated. In particular, they can even be generated by an adaptive adversary who knows the learner's algorithm and observes her decisions in previous rounds.
This generality allows us to handle strategic bidders.

\paragraph{Instantaneous truthfulness versus global truthfulness.}
The reason to consider such a second-price auction is that it is well-known to be \textit{instantaneous truthful}, that is, looking only at a particular round, each bidder is incentivized to use their true valuation as their bid~\citep{aggarwal2006truthful}.
Prior work such as~\citep{devanur2009price,ICML'23:auction} considers designing a (non-contextual) auction mechanism that is globally truthful so that even looking at all rounds together, each bidder is still incentivized to bid their true valuation.
However, this is often achieved by an explore-then-commit strategy which fixes the auction mechanism (such as the estimated CTRs) after a certain period of pure exploration.
Such strategies not only are impractical but also would not make sense at all in our contextual setting with potentially adversarial contexts.
In fact, exactly because the contexts in our problem are partially decided by exogenous factors such as the user queries (that are not manipulable by the bidders), we do not find global truthfulness a reasonable  requirement for our problem.
We thus stick with only instantaneous truthfulness and allow the learner to change the  auction  on the fly based on the context.

\paragraph{Comparisons to contextual bandits.}
One key difference between our problem and the well-studied contextual bandit problem is that we \emph{cannot} freely pick an ``arm'' (corresponding to the winner in our context), but have to do so via
proposing a particular estimated CTR vector $\wt{\rho}_t$, which itself affects the ``reward'' of an arm according to the definition of the winner's payment, making the problem more complicated.
On the other hand, what is common in both problems is that a good balance between exploration and exploitation is clearly necessary due to the limited feedback on the selected arm/winner only.
Because of these similarities and differences, in what follows we will show that some ideas from the contextual bandit literature are readily applicable to our problem, while others require more  adjustments.

\section{Achieving $\order(\sqrt{T})$ Regret Inefficiently}\label{sec:EXP4}
\begin{algorithm*}[t]
\caption{Exponential Weights for Contextual Auction}\label{alg:EXP4}

Input: learning rate $\eta>0$ and CTR predictor class $\calF$

\For{$t = 1, 2, \cdots , T$}{
	Sample a function $f_t$ from $q_t$, a distribution over $\calF$ defined via $q_{t,f} \propto  \exp(-\eta\sum_{s<t}\ellhat_{s,f})$.

    Receive context $x_t$ and the set of $N_t$ bidders.
    
     Set the estimated CTR to be $\wt{\rho}_{t,i}=f_t(x_t, i)$ for all $i\in[N_t]$ and receive bids $b_t\in [0,1]^{N_t}$.
    
    Select the bidder $i_t=\argmax_{i\in [N_t]}b_{t,i}\wt{\rho}_{t,i}$ as the winner, with payment per click $d_t=\frac{b_{t,j_t}\wt{\rho}_{t,j_t}}{\wt{\rho}_{t,i_t}}$ where $j_t = \argsmax_{i\in [N_t]} b_{t,i}\wt{\rho}_{t,i}$ is the runner up.
    
    Receive feedback $c_t=\mathbbm{1}\{\text{ad $i_t$ is clicked}\}$ and payment $c_td_t$.

    Define loss estimator $\ellhat_{t, f}$ for each $f\in\calF$ as
    \begin{equation}\label{eqn:ellhat}
        \ellhat_{t,f}= \begin{cases}
        \frac{\mathbbm{1}\{i_t=\argmax_{i\in [N_t]}b_{t,i}f(x_t,i)\}}{p_{t,i_t}}\left(1-c_t\cdot\frac{\smax_{j\in[N_t]}\ b_{t,j}f(x_t,j)}{f(x_t,i_t)}\right), &\text{\ips}\\
        \frac{1}{4\eta}(f(x_t,i_t) - c_t)^2 - \smax_{j\in [N_t]}\ b_{t,j}f(x_t,j), &\text{\optsq}
     \end{cases}
    \end{equation}
    and $p_{t,i}=\Pr_{f_t \sim q_t}\{i_t=i\}$ is the probability of $i$ being the winner.
}
\end{algorithm*}

In this section, as the first step, we show how adopting the idea of a classical contextual bandit algorithm called \expfour~\citep{auer2002nonstochastic} leads to, for example, $\order(\sqrt{NT\log|\calF|})$ regret for a finite predictor class $\calF$.
The algorithm is computationally inefficient, but it illustrates that $\sqrt{T}$-type regret is obtainable for this problem information-theoretically.

The idea of \expfour is to maintain a distribution $q_t$ over all predictors in $\calF$, defined via a classical exponential weight scheme: $q_{t,f} \propto  \exp(-\eta\sum_{s<t}\ellhat_{s,f})$ where $\ellhat_{s,f}$ is an estimator for some loss $\ell_{s, f} \in [0,1]$ of predictor $f$ at round $s$.
With such a distribution, the algorithm simply samples a predictor $f_t$ from $q_t$ for round $t$ and follows its suggestion.

In our problem, ``following $f_t$'s suggestion'' means setting the estimated CTR $\wt{\rho}_{t,i}$ directly as $f_t(x_t, i)$ for each $i \in [N_t]$.
The loss $\ell_{t, f}$ of predictor $f$ at round $t$ is intuitively the negative expected payment if one follows $f$'s suggestion.
To ensure a range of $[0,1]$, we shift it by $1$, leading to $\ell_{t, f} = 1 - \rho_{t, i_{t,f}}\frac{\smax_j b_{t,j}f(x_t,j)}{f(x_t,i_{t,f})}$ where $i_{t,f} = \argmax_i b_{t,i} f(x_t,i)$.
It remains to construct the loss estimator $\ellhat_{t,f}$ based on the learner's observations $b_t$ and $c_t$.
Even though our loss structure is quite different from contextual bandits, we find the standard inverse propensity score weighting still applicable, leading to a natural loss estimator defined in the \ips option of ~\pref{eqn:ellhat}.
It is clear that $\ellhat_{t,f}$ is unbiased since $\E[\mathbbm{1}\{i_t = i_{t,f}\}] = p_{t,i_{t,f}}$ and $\E[c_t] = \rho_{t, i_t}$ implies $\E[\ellhat_{t,f}] = \ell_{t,f}$.
See~\pref{alg:EXP4} for the complete pseudocode. 
Following standard analysis of \expfour, we show the following regret bound when $\calF$ is finite (whose proof is deferred to \pref{app:EXP4}).
\begin{restatable}{theorem}{expFour}\label{thm:exp4}
\pref{alg:EXP4} with learning rate $\eta=\sqrt{\frac{\log|\calF|}{\sum_{t=1}^TN_t}}$ and \ips estimators guarantees  
$
    \E\left[\Reg\right] = \order\Big(\sqrt{\sum_{t=1}^TN_t\log|\calF|}\Big) = \order(\sqrt{NT\log|\calF|}).
$\footnote{For simplicity, we set the learning rate in terms of the unknown quantity $\sum_{t=1}^TN_t$, but this can be easily resolved by applying a standard doubling trick. The same holds for other results in this work.}
\end{restatable}

Now, we argue that such $\sqrt{T}$ dependence in the regret is unavoidable.
To see this, we first discuss the implication of \pref{thm:exp4} for a special non-contextual case where for each $i$, $\rho_{t,i} = \rho_i$ stays the same for all $t$ (so the context $x_t$ plays no role in predicting the CTRs), which is essentially the setting of~\citep{devanur2009price}.
In this case, a class of constant predictors (that ignore the context input) $\calF=\{(x,i) \rightarrow \theta_i : \theta \in [0,1]^N\}$ trivially satisfies \pref{asm:realizability}.
By applying \pref{alg:EXP4} to a discretized version this class: $\wh{\calF}=\{(x,i) \rightarrow \theta_i : \theta \in \{0, \frac{1}{T}, \frac{2}{T}, \ldots, 1\}^N\}$, which has a size of $\order(T^N)$ and can approximate $\calF$ up to error $\order(1/T)$, we immediately obtain the following corollary (see \pref{app:EXP4} for the proof).

\begin{restatable}{corollary}{noncontextual}\label{cor:noncontextual}
    In the non-contextual setting described above, \pref{alg:EXP4} with predictor class $\wh{\calF}$, learning rate $\eta=\sqrt{\frac{\log T}{T}}$, and \ips estimators guarantees $\E[\Reg]\leq \order(N\sqrt{T\log T})$.\footnote{We note that this is not a contradiction with the $\Omega(T^{2/3})$ lower bound of~\citep{devanur2009price} since they insist on global truthfulness while we do not; see related discussions in \pref{sec: pre}.}
\end{restatable}

Notably, \citep{ICML'23:auction} considers the same non-contextual setting and designs a UCB-based algorithm with $\order\big(\sum_{i=1}^N\frac{1}{\rho_i}\sqrt{T\log (NT)}\big)$ regret. Our result thus strictly improves upon theirs by removing the dependence on $\frac{1}{\rho_i}$, which can be arbitrarily large as long as there exists an ad with a very low CTR.

Next, we show in the following theorem that the $\sqrt{T}$ dependence is unavoidable even in this non-contextual case. 
\begin{restatable}{theorem}{noncontextual_lower_bound}\label{thm:noncontextual_lower_bound}
    In the non-contextual setting described above with $T\geq N\geq 3$, for any algorithm $\calA$, there exists a sequence of bids $\{b_{t}\}_{t=1}^T$ and CTRs $\{\rho_{i}\}_{i=1}^N$ such that the expected regret suffered by $\calA$ is at least $\Omega(\sqrt{NT})$.
\end{restatable}
The formal proof of \pref{thm:noncontextual_lower_bound} is deferred to \pref{app:EXP4}. The general idea of the proof is to transform the auction problem in this non-contextual setting to a hard instance of the classical multi-armed bandit problem, whose minimax regret is well-known to be of order $\Theta(\sqrt{NT})$. 
To see this, consider the case where $N_t = N$ and $b_t = \mathbf{1}$ for all $t$, and $\rho_{i} = 1/2$ for all $i$ except for two uniformly at random chosen ads, whose CTR are always $1/2+\Theta(\sqrt{N/T})$.
Obverse that: 1) the oracle strategy with $\wt{\rho}_{t,i} = \rho_{i}$ receives expected payment $1/2+\Theta(\sqrt{N/T})$ per round; 2) no matter which ad (arm) is selected as the winner by the learner, since the payment per click can not exceed the  bid $1$ according to the auction design, the expected payment of the learner is at most the CTR of the chosen winner (arm); 3) the feedback to the learner at each round is only a Bernoulli random variable with mean being the CTR of the chosen winner (arm).
These facts together show that the instance above is no easier than an $N$-armed bandit problem with reward means being the same as the CTR configuration above, which is well-known to incur $\Omega(\sqrt{NT})$ regret.

While \pref{alg:EXP4} with \ips achieves optimal regret guarantee, as mentioned, the caveat of \pref{alg:EXP4} with \ips loss estimators is its computational inefficiency.
Indeed, the complicated form of the estimator, in particular the $p_{t,i_t}$ term, makes it difficult to compute efficiently.
In the next two sections, we address this issue via two different approaches.

\section{An Efficient Approach via Optimistic Squared Errors}\label{sec: TS}

Our first approach to derive a practically efficient algorithm is to replace the \ips estimator $\ellhat_{t,f}$ in \pref{alg:EXP4} with an optimistic squared error that not only is easy to compute itself, but also allows efficient gradient computation so that one can directly apply stochastic gradient Langevin dynamics (SGLD)~\citep{welling2011bayesian} to approximately sample from $q_t$.
The optimistic squared error is defined in \optsq of \pref{eqn:ellhat}, whose design is largely inspired by the Feel-Good Thompson Sampling algorithm of a recent work~\citep{zhang2022feel} for contextual bandits.
It contains a natural squared error term $(f(x_t, i_t) - c_t)^2$ (scaled by $\nicefrac{1}{4\eta}$), which measures how good $f$ is in predicting the CTR of the selected winner $i_t$ under context $x_t$.
However, the squared error itself is insufficient in encouraging exploration; see \citep{zhang2022feel} for detailed discussions even for the easier contextual bandit problem.
To ensure exploration, we propose to subtract from the squared error a ``bonus'' term $\smax_j b_{t,j} f(x_t,j)$, which is the expected payment we would receive if $f$'s predictions were exactly the true CTRs (and we followed these predictions in setting $\wt{\rho}_t$).
This serves as a form of optimism: the larger the payment in $f$'s predicted world, the smaller its optimistic squared error $\ellhat_{t,f}$, and consequently the larger weight $f$ gets in the sampling distribution $q_t$ to ensure that it is properly explored.
The necessity of such optimism is not only for regret analysis, but also verified in our experiments (see \pref{sec:experiment}).

\paragraph{Efficient (approximate) implementation.}
The advantage of the \optsq estimator over \ips is that it enables efficient sampling for a parametrized predictor class.
Specifically, consider a parametrized and differentiable predictor class $\calF = \{f_\theta : \theta \in \Theta\}$ for some $d$-dimensional parameter space $\Theta \subseteq \R^d$.
Then the gradient of $\ellhat_{t,f_\theta}$ with respect to $\theta$ (at any differentiable point) is 
\[
\frac{\partial \ellhat_{t,f_\theta}}{\partial \theta} =
\frac{1}{2\eta}(f_\theta(x_t, i_t) - c_t)\cdot \frac{\partial f_\theta(x_t, i_t)}{\partial \theta} - b_{t, k} \frac{\partial f_\theta(x_t, k)}{\partial \theta}
\]
where $k= \argsmax b_{t,j} f_\theta(x_t, j)$.
We can then apply SGLD to approximately sample from $q_t$ as follows: start from a random initialization $\theta \in \Theta$, then repeatedly sample uniformly at random an $s\in [t-1]$ and update
\begin{equation}\label{eqn:SGLD}
\theta \leftarrow \theta - \alpha\eta\frac{\partial \ellhat_{s,f_\theta}}{\partial \theta} + \sqrt{\frac{2\alpha}{t}}\epsilon
\end{equation}
where $\epsilon$ is a fresh $d$-dimensional standard Gaussian noise and $\alpha$ is a step size.
Even though the existing theory for SGLD does not necessarily tell us how many steps are needed to ensure an accurate enough sample (which is already the case for contextual bandits~\citep{zhang2022feel}), this is clearly a practically efficient gradient-based method in line with standard machine learning practice. 

\paragraph{Regret guarantees.}
Next, we show that the resulting algorithm also enjoys similar regret guarantees as the version with \ips. We start by the following general theorem, whose analysis extends the idea of~\citep{zhang2022feel} to our more complicated loss structure and the new bonus term that is tailored to second-price auctions and importantly involves possibly adversarial bids $b_t$ (see \pref{app: TS}).

\begin{restatable}{theorem}{TS}\label{thm:TS}
    \pref{alg:EXP4} with learning rate $\eta\leq 1$ and \optsq estimators ensures
    $
        \E[\Reg] \leq  \frac{Z_T}{\eta} + \order\left(\eta\sum_{t=1}^TN_t\right),
    $
    where $Z_T = -\E\left[\log\E_{f\sim q_1}\exp\left(-\eta\sum_{t=1}^T\left(\ellhat_{t,f}-\ellhat_{t,f^*}\right)\right)\right]$. 
\end{restatable}
Here, $Z_T$ should be treated as some complexity measure of the predictor class $\calF$.
As a concrete example, when $\calF$ is finite, we have the following corollary that exactly matches  \pref{thm:exp4}.

\begin{restatable}{corollary}{TS_finite}\label{cor:TS_finite}
For a finite $\calF$, we have $Z_T \leq \log |\calF|$, and thus the regret bound in \pref{thm:TS} becomes $\E[\Reg] = \order\Big(\log|\calF|+\sqrt{\sum_{t=1}^T N_t \log|\calF|}\Big)$ after picking the optimal $\eta$.
\end{restatable}
\begin{proof}
Since $q_1$ is uniform, we have
\begin{align*}
    Z_T &= -\E\left[\log\sum_{f\in\calF}\frac{1}{|\calF|}\exp\left(-\eta\sum_{t=1}^T\left(\ellhat_{t,f}-\ellhat_{t,f^*}\right)\right)\right]\\
    &\leq -\E\left[\log \frac{1}{|\calF|}\exp\left(-\eta\sum_{t=1}^T\left(\ellhat_{t,f^*}-\ellhat_{t,f^*}\right)\right)\right]  = \log|\calF|,
\end{align*}
finishing the proof.
\end{proof}

As another example, we consider a Lipschitz class in the next result (proof deferred to \pref{app: TS}).
\begin{restatable}{corollary}{TSLipschitz}\label{cor:TS_Lipschitz}
For some constants $\alpha, B \geq 0$ and a parametrized class $$\calF = \{f_\theta: \theta \in \Theta \subseteq [-B,B]^d, \text{$f_\theta$ is $\alpha$-Lipschitz with respect to $\|\cdot\|_\infty$ in $\theta$}\},$$ we have $Z_T = \order(\alpha +d\log BT)$, and thus the regret bound in \pref{thm:TS} becomes $$\E[\Reg] = \order\left(\sqrt{(\alpha+d\log(BT))\sum_{t=1}^TN_t}\right)$$ after picking the optimal $\eta$.  
\end{restatable}

\section{An Efficient Approach via a Regression Oracle}\label{sec: squarecb}

The second approach we take to derive a practically efficient algorithm follows a different trend of recent studies that reduce contextual bandits to some online regression problem and only access the predictor class $\calF$ through an efficient regression oracle~\citep{foster2020beyond,foster2021efficient,zhu2022contextual}.
More concretely, we assume access to an online regression oracle \AlgSq which follows the following protocol: at each round $t \in [T]$, the oracle selects randomly a function $f_t\in \calF$, then it
receives a tuple $(x_t, i_t, c_t) \in \calX \times [N] \times [0,1]$, potentially generated by an adaptive adversary, and suffers a squared  error $(f_t(x_t,i_t) - c_t)^2$.\footnote{For simplicity, we consider a \textit{proper} oracle here, but one can also use an \textit{improper} oracle that makes a prediction after seeing $(x_t,i_t)$, not necessarily following some $f_t\in\calF$, as in~\citep{foster2020beyond}.}
We assume that the oracle ensures some regret bound against the best predictor in $\calF$:

\begin{assumption}[Bounded Squared Error Regression Regret]
\label{asm:regression_oracle}
The regression oracle \AlgSq guarantees that for any (potentially adaptively chosen) sequence $\{(x_t, i_t, c_t)\}_{t\in[T]}$, the following is bounded by $\RegSq$:
\begin{equation*}
\begin{aligned}
\E\Big[\sum_{t=1}^T \left(f_t(x_t, i_t) - c_t\right)^2-\inf_{f \in \calF}\sum_{t=1}^T \left(f(x_t, i_t) - c_t\right)^2\Big]. 
\end{aligned}
\end{equation*}
\end{assumption}
There are many examples of such a regression oracle.
For instance, when $\calF$ is a $d$-dimensional linear class, Online Newton Step (ONS)~\citep{hazan2007logarithmic} achieves $\RegSq=\order(d\log T)$. We refer the reader to~\citep{foster2020beyond} for other examples, and point out that in practice (and in our experiments), such an oracle is often implemented by a simple gradient-based method.
The important point is that the regression problem does not require a balance between exploration and exploitation, and is thus generally easier than the original problem and is better studied with known efficient algorithms.
This also means that to reduce the original problem to its regression counterpart, the key is to figure out an appropriate exploration strategy.

\citet{foster2021statistical} provide a general framework to find the optimal exploration strategy using a concept called \textit{Decision-Estimation Coefficient} (DEC).
For our problem, it boils down to understanding the following quantity
\begin{equation}\label{eqn:dec}
\begin{split}
&\dec(\wh{\rho}) = \min_{Q}\max_{\rho \in [0,1]^N,b \in [0,1]^N}\\
&\qquad\E_{\wt{\rho}\in Q}\Bigg[ \underbrace{\left(\smax_{i\in[N]} b_i \rho_i - \rho_{i^*}\frac{\smax_{j \in [N]} b_j \wt{\rho}_j }{\wt{\rho}_{i^*}} \right)}_{\text{per-round regret}} - \gamma\underbrace{\left(\rho_{i^*} - \wh{\rho}_{i^*} \right)^2}_{\text{squared error}}\Bigg]
\end{split}
\end{equation}
where the $\min$ is over all possible distributions $Q$ over the set $[0,1]^N$, $i^* = \argmax_i b_i\wt{\rho}_i$ is the winner according to $\wt{\rho}$, $\gamma > 0$ is some coefficient, and $\wh{\rho} \in [0,1]^N$ is a given CTR prediction (provided by the regression oracle).
In words, $\dec(\wh{\rho})$ measures how small the gap can be made between the per-round regret of the learner using a randomized strategy $Q$ and the squared prediction error of the oracle in the worst case.
Understanding this quantity provides both an algorithm and a corresponding regret bound, as shown below.
\begin{proposition}\label{prop:dec}
Suppose that at each round $t$, 1) the oracle \AlgSq outputs a CTR predictor $f_t\in\calF$; 2) the learner then randomly sets $\wt{\rho}_t$ according to the distribution that realizes the $\min_Q$ in the definition of $\dec(\wh{\rho}_t)$ where $\wh{\rho}_{t,i}=f_t(x_t, i)$; 3) finally the oracle is fed with the tuple $(x_t, i_t, c_t)$.
Then the regret of the learner satisfies $\E[\Reg] \leq \gamma \RegSq + \E\left[\sum_{t=1}^T \dec(\wh{\rho}_t)\right]$.
\end{proposition}

For example, if $\dec(\wh{\rho}_t)$ turns out to be $\order(\frac{N}{\gamma})$ for all $t$, then picking the optimal $\gamma$ leads to $\E[\Reg]= \order(\sqrt{NT \RegSq})$ (which, for linear class, implies $\E[\Reg]= \order(\sqrt{dNT\log T})$ if ONS is used as the oracle).
For contextual bandits, the corresponding DEC is indeed shown to be $\order(\frac{N}{\gamma})$~\citep{foster2020beyond}.

\paragraph{An upper bound and a simple algorithm.}
Unfortunately, we are unable to show that our DEC is of order $1/\gamma$ due to the very complicated second-price structure.
Nevertheless, we can show a worse bound $\dec(\wh{\rho}_t) = \order(\sqrt{\nicefrac{N}{\gamma}})$. 
This is achieved by a simple $\epsilon$-greedy strategy (with $\epsilon = \sqrt{\nicefrac{N}{\gamma}}$) --- let $Q$ concentrate on $N+1$ CTR predictions: the greedy one $\wh{\rho}$ with probability $1-\epsilon$ (exploitation), and the one-hot CTR predictions $e_1, \ldots, e_N$, each with probability $\epsilon/N$ (exploration, since when $\wt{\rho} = e_i$, ad $i$ will always be selected as the winner, and we observe signal on its CTR).\footnote{Note that in this case the payment is always $0$, which, practically speaking,  might not be desirable. However, since  in reality there is always a minimum allowed bid $\sigma > 0$, we can in fact replace $e_i$ by $\frac{\sigma}{2}\cdot\one+(1-\frac{\sigma}{2})e_i$ instead, which still ensures exploration of ad $i$ while leading to at least $\sigma^2/2$ payment per click.}
Put together, this leads to \pref{alg:squareCB.A} and the following result, whose analysis  requires a careful treatment to the second-price structure despite the simplicity of the algorithm (see \pref{app: squarecb}).

\begin{algorithm*}[t]
\caption{Epsilon Greedy with an Online Regression Oracle}\label{alg:squareCB.A}

Input: exploration parameter $\epsilon>0$ and an online regression oracle \AlgSq satisfying \pref{asm:regression_oracle}

\For{$t=1,2,\dots,T$}{
    Obtain a CTR predictor $f_t$ from the oracle $\AlgSq$. 
    
    Receive context $x_t$ and the set of $N_t$ bidders.

    With $\epsilon$ probability, pick $i\in[N_t]$ uniformly at random and set $\wt{\rho}_t=e_i\in\R^{N_t}$; with the remaining $1-\epsilon$ probability, set $\wt{\rho}_{t,i}=f_t(x_t,i)$ for all $i\in[N_t]$
    
    Receive bids $b_t\in[0,1]^{N_t}$.
    
    Select the bidder $i_t=\argmax_{i\in[N_t]}b_{t,i}\wt{\rho}_{t,i}$ as the winner, with payment per click  $d_t=\frac{b_{t,j_t}\wt{\rho}_{t,j_t}}{\wt{\rho}_{t,i_t}}$ where $j_t = \arg\smax_{i\in[N_t]}b_{t,i}\wt{\rho}_{t,i}$ is the runner-up.

    Receive feedback $c_t=\mathbbm{1}\{\text{ad $i_t$ is clicked}\}$ and payment $c_td_t$.
    
    Feed the tuple $(x_t,i_t,c_t)$ to the oracle $\AlgSq$.
}
\end{algorithm*}
\begin{restatable}{theorem}{squareCB}\label{thm:squareCB}
For any $\wh{\rho}_t \in [0,1]^N$, we have $\dec(\wh{\rho}) = \order(\sqrt{\nicefrac{N}{\gamma}})$, evidenced by the $\epsilon$-greedy strategy described above.
Consequently,  \pref{alg:squareCB.A} with $\epsilon=T^{-\frac{1}{3}}(N\RegSq)^{\frac{1}{3}}$ 
guarantees that
$
    \E[\Reg] 
    = \order(T^{\frac{2}{3}}(N\RegSq)^{\frac{1}{3}}).
$
\end{restatable}

\paragraph{Discussion and conjecture.}
Despite the worse regret bound, one nice property of \pref{alg:squareCB.A} is that it in fact never uses the bid information in deciding $\wt{\rho}_t$, since the oracle never receives the bid information as feedback.
This is in sharp contrast with \pref{alg:EXP4} where both \ips and \optsq are defined in terms of the bids.
This property is useful when the second-price auction is actually run by a third party who collects the bids from the advertisers without revealing them to the platform that makes the CTR predictions.

In fact, this property is inherent in this DEC approach and not just because of our choice of a potentially suboptimal $\epsilon$-greedy strategy.
Indeed, even for the algorithm described in \pref{prop:dec} that solves the $\min_Q$ part of \pref{eqn:dec} exactly, it still by definition does not use the bid information but only the predictions from the oracle (which, again, is independent of any bids from prior rounds).
We conjecture that because of this aspect, the optimal bound of $\dec(\wh{\rho})$ might indeed be $\Theta(\sqrt{N/\gamma})$, achieved by the simple $\epsilon$-greedy strategy.
We leave this question for future investigation.

\section{Experiments}\label{sec:experiment}

We implement the two efficient algorithms we propose, namely~\pref{alg:EXP4} with \optsq  loss estimators and~\pref{alg:squareCB.A}, and demonstrate their superior performance on a synthetic dataset compared to three simple baselines: 1) a strategy that always make random CTR predictions, 2) a strategy that always make the same CTR prediction for all ads, which is equivalent to running a second-price auction using only the bids, and 3) a variant of \pref{alg:EXP4} that uses the \optsq estimator but removes its optimistic part $\smax_j f(x_t,j)b_{t,j}$ (denoted by $\sq$).

For \pref{alg:EXP4}, $\eta$ is chosen from $\{\frac{1}{16}, \frac{1}{8}, \frac{1}{4}\}$, and SGLD \pref{eqn:SGLD} is run for $32$ steps per round with a step size $\alpha$ chosen from $\{0.0005,0.001,0.005,0.01,0.05\}$. For \pref{alg:squareCB.A}, the parameter $\epsilon$ is chosen from $\{T^{-\frac{1}{3}}, 2T^{-\frac{1}{3}}, 4T^{-\frac{1}{3}}\}$, and the regression oracle is simply implemented by online gradient descent~\citep{zinkevich2003online} with a learning rate from $\{0.001,0.005\}$.

We use one-hidden-layer neural nets for the CTR predictors.
Specifically, at each time $t$, our context $x_t$ is a matrix in $\R^{d \times (1+N_t)}$, where the
first column $x_{t,0}$ represents some common information shared by all advertisers (such as the user query), and the $(i+1)$-th column $x_{t,i}$ represents the feature of ad $i$.
Our predictor class is then $\calF = \{f_\theta: f_\theta(x,i)=\sigmoid(\theta_0^\top x_0 + \theta^\top x_i), \theta_0, \theta\in [-1,1]^{d}\}$
where for any $u\in \R$, $\sigmoid(u)=\frac{1}{1+e^{-u}}$.

\paragraph{Synthetic dataset construction.}
We generate a synthetic dataset by choosing $d=128$ and $T=10^4$ and uniformly at random sampling $N_t$ from $5$ to $10$, $x_t$ from $[-1,1]^{d\times (1+N_t)}$, and $b_t$ from $[0.1,1]$.
To generate the underlying CTR predictor $f^*$, we first generate some fake CTRs uniformly at random from $[0.2,1]$ for each $t$ and then use the full dataset to fit a model from $\calF$ as the final $f^*$.
To avoid the trivial fixed CTR strategy already performing very well, we also make a final adjustment and set the bid of the ad with the lowest CTR to be $1$ (imagine a little-known new brand trying to bid high to get more impressions).

\paragraph{Results.}
We repeat our experiment with $4$ different random seeds 
and plot in \pref{fig:regret_synthetic} the average and the standard deviation (as shaded areas) of the cumulative regret over these $4$ trials under the best hyperparameters. From \pref{fig:regret_synthetic}, we see that 
1) as our theory indicates, \pref{alg:EXP4} \optsq indeed suffers lower regret than \pref{alg:squareCB.A}; 
2) both of our algorithms beat the three baselines with a large margin;
3) \pref{alg:EXP4} \sq is only comparable to the random CTR strategy, demonstrating that the optimistic part of \optsq is indeed both theoretically and empirically critical.

\begin{figure}[t]
  \centering  \includegraphics[width=0.45\textwidth]{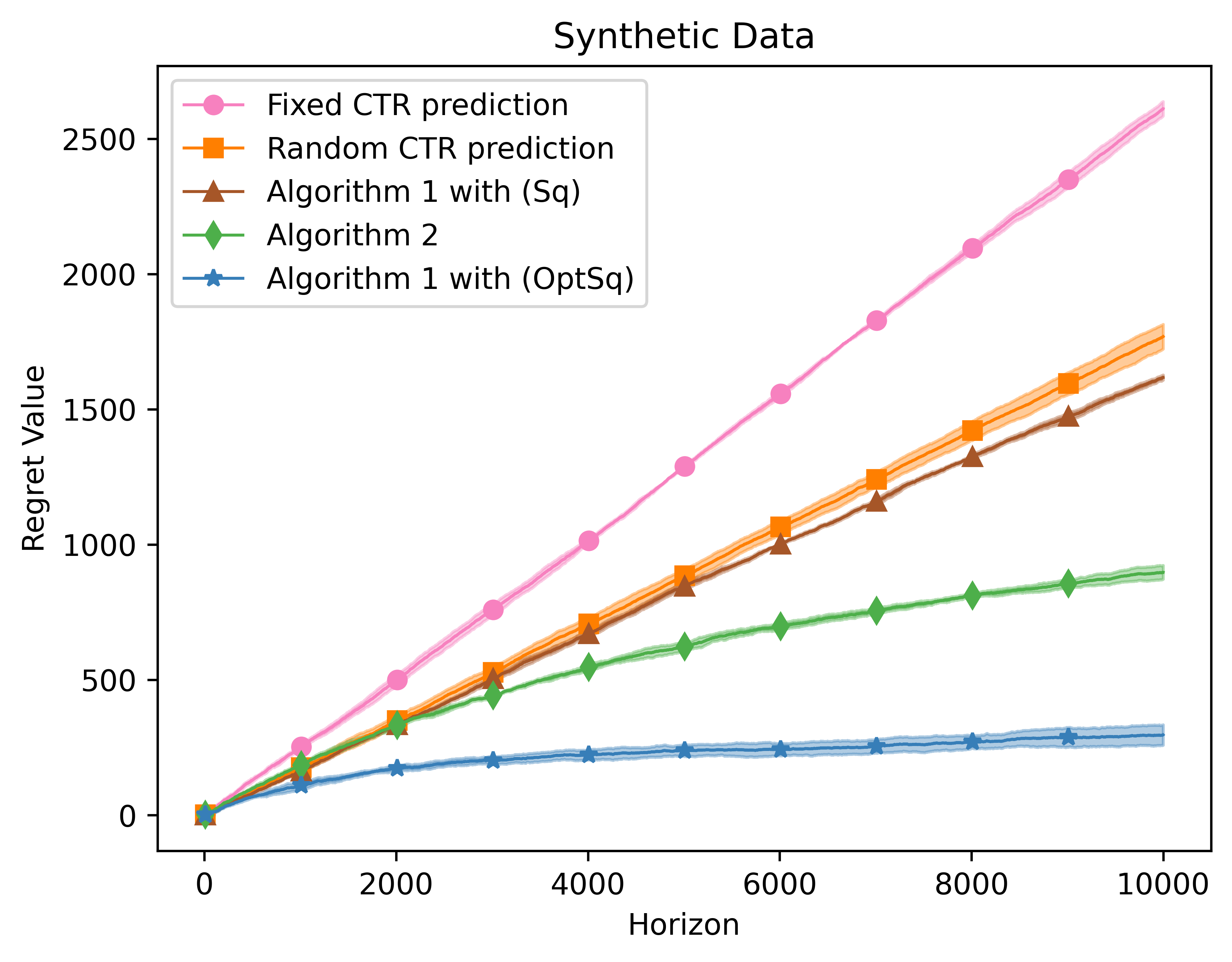}
  \caption{Performance of different algorithms on synthetic data}
  \label{fig:regret_synthetic}
\end{figure}

\bibliography{bibfile}
\bibliographystyle{plainnat}

\newpage
\appendix
\section{Omitted Proofs in~\pref{sec:EXP4}}\label{app:EXP4}
\begin{proof}[Proof of \pref{thm:exp4}]
    According to standard analysis of exponential weight (see for example~\citep[Lemma~14]{wei2020taking}), we have for any $f_0 \in \calF$:
    \begin{align*}
        \inner{q_t-e_{f_0}, \ellhat_t} &\leq \frac{1}{\eta} D(e_{f_0}, q_t) - \frac{1}{\eta} D(e_{f_0}, q_{t+1}) + \frac{\eta}{2}\sum_{f\in \calF}q_{t,f}\ellhat_{t,f}^2,
    \end{align*}
    where $D(u,v)=\sum_{f\in\calF} u_f\log\frac{u_f}{v_f}$ is the KL divergence between $u$ and $v$. Taking summation over $t\in [T]$, we have with $i_{t,f}=\argmax_ib_{t,i}f(x_t,i)$:
    \begin{align*}
        &\sum_{t=1}^T\inner{q_t-e_{f_0}, \ellhat_t} \\
        &\leq \frac{D(e_{f_0}, q_1)}{\eta} + \frac{\eta}{2}\sum_{t=1}^T\sum_{f\in \calF}q_{t,f}\ellhat_{t,f}^2 \\
        &= \frac{D(e_{f_0}, q_1)}{\eta} + \frac{\eta}{2}\sum_{t=1}^T\sum_{i\in [N_t]}\sum_{f\in \calF, i=i_{t,f}}q_{t,f}\ellhat_{t,f}^2 \\
        &= \frac{D(e_{f_0}, q_1)}{\eta} + \frac{\eta}{2}\sum_{t=1}^T\sum_{i\in[N_t]}\sum_{f\in \calF, i=i_{t,f}}\frac{q_{t,f}}{p_{t,i}^2}\left(1-c_t\cdot\frac{\smax_j b_{t,j}f(x_t,j)}{f(x_t, i_{t,f})}\right)^2\cdot\mathbbm{1}\{i_t=i\} \\
        &\leq \frac{D(e_{f_0}, q_1)}{\eta} + \frac{\eta}{2}\sum_{t=1}^T\sum_{i\in[N_t]}\sum_{f\in\calF, i=i_{t,f}}\frac{q_{t,f}\mathbbm{1}\{i_t=i\}}{p_{t,i}^2} \\
        &= \frac{D(e_{f_0}, q_1)}{\eta} + \frac{\eta}{2}\sum_{t=1}^T\sum_{i\in[N_t]}\frac{\mathbbm{1}\{i_t=i\}}{p_{t,i}} \tag{$p_{t,i} = \sum_{f\in\calF: i=i_{t,f}}q_{t,f}$}
    \end{align*}
    where the inequality is because $\smax_j b_{t,j}f(x_t,j) \leq b_{t,i_{t,f}}f(x_t, i_{t,f})$ and thus $1-c_t\cdot\frac{\smax_j b_{t,j}f(x_t,j)}{f(x_t, i_{t,f})} \in [0,1]$.  
    Taking expectation over both sides and noticing 
    \begin{align*}
        &\E\left[\frac{\mathbbm{1}\{i_t=i\}}{p_{t,i}}\right] = \E\left[\frac{p_{t,i}}{p_{t,i}}\right] = 1,\\
        \E\left[\ellhat_{t,f}\right] 
        &= \E\left[\frac{\mathbbm{1}\{i_t=i_{t,f}\}}{p_{t,i_t}}\left(1-\rho_{t,i_{t,f}}\cdot\frac{\smax_j b_{t,j}f(x_t,j)}{f(x_t,i_{t,f})}\right)\right] \\
        &= 1 - \rho_{t,i_{t,f}}\cdot \frac{\smax_j b_{t,j}f(x_t,j)}{f(x_t,i_{t,f})} \\ 
        &= 1 - f^*(x_t,i_{t,f})\cdot \frac{\smax_j b_{t,j}f(x_t,j)}{f(x_t,i_{t,f})}, 
    \end{align*}
    we obtain
    \begin{equation}\label{eqn:exp4-regret}
    \begin{split}
    &\E\left[\sum_{t=1}^T f^*(x_t,i_{t,f_0})\cdot \frac{\smax_j b_{t,j}f_0(x_t,j)}{f_0(x_t,i_{t,f_0})}  - \sum_{t=1}^T f^*(x_t,i_{t,f_t})\cdot \frac{\smax_j b_{t,j}f_t(x_t,j)}{f_t(x_t,i_{t,f_t})} \right]   \\
     &= \E\left[\sum_{t=1}^T f^*(x_t,i_{t,f_0})\cdot \frac{\smax_j b_{t,j}f_0(x_t,j)}{f_0(x_t,i_{t,f_0})}  - \sum_{t=1}^T c_td_t \right]  \leq \frac{D(e_{f_0},q_1)}{\eta} + \frac{\eta}{2}\sum_{t=1}^TN_t.
    \end{split}\end{equation}
    Picking $f_0 = f^*$, which is in the class $\calF$ by \pref{asm:realizability}, the left-hand side above becomes exactly the expected regret $\E\left[\Reg\right]$.
    Finally, since $q_1$ is a uniform distribution over $\calF$, we know that $D(e_{f^*},q_1)=\log|\calF|$, and picking $\eta=\sqrt{\frac{\log|\calF|}{\sum_{t=1}^TN_t}}$ then leads to the claimed regret bound.
\end{proof}

\begin{proof}[Proof of \pref{cor:noncontextual}]
    For any $u\in[0,1]$, define $\lfloor u\rfloor_T = \frac{1}{T}\lfloor Tu\rfloor$. Since we consider the non-contextual setting, we have $\rho=\rho_t$ for all $t\in[T]$ and the policy $f\in\wh{\calF}$ can be represented as a vector in $[0,1]^N$.
    To show that~\pref{alg:EXP4} guarantees $\order(N\sqrt{T\log N})$ regret, let $\wh{f}^*=(\lfloor \rho_1\rfloor_T, \dots,\lfloor \rho_N\rfloor_T)\in \wh{\calF}$. Using~\pref{eqn:exp4-regret} and picking $f_0=\wh{f}^*$, we know that
    \begin{align*}
    \E\left[\sum_{t=1}^T \rho_{i_{t,\wh{f}^*}} \cdot \frac{b_{t,j_{t,\wh{f}^*}}\wh{f}^*_{j_{t,\wh{f}^*}}}{\wh{f}^*_{i_{t,\wh{f}^*}}}  - \sum_{t=1}^T c_t d_t \right] \leq \order\left(\frac{N\log T}{\eta}\right) + \frac{\eta NT}{2},
    \end{align*}
    where $i_{t,\rho}=\argmax_ib_{t,i}\rho_{t,i}$ and $j_{t,\rho}=\argsmax_ib_{t,i}\rho_{t,i}$. 
    It remains to argue that the first term above is close to $\sum_{t=1}^T \smax_i b_{t,i}\rho_i$, the revenue of the oracle strategy.    
    To this end, for each time $t$ we consider two cases depending on whether $\wh{f}^*$ and $\rho$ selects the same winner.
    If $i_{t,\wh{f}^*}=i_{t,\rho}$, 
    we have
    \begin{align*}
        \rho_{i_{t,\wh{f}^*}}\cdot\frac{b_{t,j_{t,\wh{f}^*}}\wh{f}^*_{j_{t,\wh{f}^*}}}{\wh{f}^*_{i_{t,\wh{f}^*}}}  &\geq b_{t,j_{t,\wh{f}^*}}\wh{f}^*_{j_{t,\wh{f}^*}} \tag{since $\wh{f}^*_i = \lfloor \rho_i \rfloor_T \leq \rho_i$} \\
        & \geq b_{t,j_{t,\rho}}\wh{f}^*_{j_{t,\rho}} \tag{since $b_{t,j_{t,\wh{f}^*}}\wh{f}^*_{j_{t,\wh{f}^*}}=\smax_ib_{t,i}\wh{f}^*_i$ and $j_{t,\rho} \neq i_{t,\rho} = i_{t,\wh{f}^*}$}\\
        & \geq b_{t,j_{t,\rho}}\rho_{j_{t,\rho}} - \frac{1}{T} \tag{since $\wh{f}^*_i = \lfloor \rho_i \rfloor_T \geq \rho_i - \frac{1}{T}$}\\
        & = \smax_ib_{t,i}\rho_i - \frac{1}{T}.
    \end{align*}
    Otherwise, we know that $b_{t,j_{t,\wh{f}^*}}\wh{f}^*_{j_{t,\wh{f}^*}} \geq b_{t,i_{t,\rho}}\wh{f}^*_{i_{t,\rho}}\geq b_{t,i_{t,\rho}}\rho_{i_{t,\rho}} - \frac{1}{T}\geq \smax_ib_{t,i}\rho_i - \frac{1}{T}$. Combining both cases, we know that 
    \begin{align*}
    &\E\left[\sum_{t=1}^T\smax_ib_{t,i}\rho_i - \sum_{t=1}^T c_t d_t\right] 
    \leq 
\E\left[\sum_{t=1}^T  \rho_{i_{t,\wh{f}^*}}\cdot\frac{b_{t,j_{t,\wh{f}^*}}\wh{f}^*_{j_{t,\wh{f}^*}}}{\wh{f}^*_{i_{t,\wh{f}^*}}} - \sum_{t=1}^T c_t d_t\right] +1\\
&\leq \order\left(\frac{N\log T}{\eta}\right) + \frac{\eta NT}{2} +1.
    \end{align*}
    Picking $\eta=\sqrt{\frac{\log T}{T}}$ achieves the claimed bound.
\end{proof}

\begin{proof}[Proof of \pref{thm:noncontextual_lower_bound}]
The proof follows the idea of the lower bound construction in multi-armed bandits. Consider an instance of the auction problem with $N_t=N$ bidders and CTR $\{\rho_i\}_{i=1}^N$. The bid vector $b_t$ at each round $t$ is $\one$. Denote $\rho^{(i,j)}$ $(i\neq j)$ to be the environment where $\rho_{k}^{(i,j)}=\frac{1}{2}$ for all $k\in [N]\backslash\{i,j\}$ and $\rho_{i}^{(i,j)}=\rho_j^{(i,j)}=\frac{1}{2}+\epsilon$ for some $\epsilon>0$ to be specified later. We also denote $\rho^{(0)}$ to be the environment where $\rho^{(0)}_i=\frac{1}{2}$ for all $i\in [N]$. 

Consider the environment $\calE$ where $\rho$ is uniformly drawn from the $\frac{N(N-1)}{2}$ environments $\{\rho^{(i,j)}\}_{i\neq j, i,j\in[n]}$. For notational convenience, we denote $\E_{i,j}[\cdot]$ ($\E_0$) to be $\E_{\rho^{(i,j)}}[\cdot]$ ($\E_{\rho^{(0)}}$). Let $n_i$ be the number of rounds ad $i$ is selected as the winner (via picking an estimated CTR). Since the payment per click can not exceed the bid $1$ according to the auction design, the expected revenue of the learner is at most the CTR of the winning ad. In addition, the benchmark revenue in environment $\rho^{(i,j)}$ is $(\frac{1}{2}+\epsilon)T$ by picking the estimated CTR as the true CTR. Therefore, the expected regret with respect to environment $\calE$ is lower bounded as follows:
\begin{align}
    \E_{\calE}[\Reg] &\geq  \frac{2}{N(N-1)}\sum_{i\neq j}\E_{i,j}\left[\left(\frac{1}{2}+\epsilon\right)T -\frac{1}{2}\sum_{k\neq i,j}n_k - \left(\frac{1}{2}+\epsilon\right)(n_i+n_j)\right] \nonumber \\
    &= \frac{2}{N(N-1)}\sum_{i\neq j, i,j\in[N]}\E_{i,j}\left[\epsilon(T-n_i-n_j)\right], \label{eqn:reg_lower_bound}
\end{align}
where the last equality uses the fact that $\sum_{i=1}^Nn_i=T$.
According to Exercise 15.2.(a) of~\citep{lattimore2020bandit}, we have
\begin{align*}
    \E_{i,j}[n_i+n_j]\leq \E_0[n_i+n_j] + T\sqrt{\frac{1}{4}\epsilon^2\E_0[n_i+n_j]}.
\end{align*}
Taking summation over all $i\neq j$, $i,j\in[N]$, we obtain that
\begin{align}
    \sum_{i\neq j, i,j\in[N]}\E_{i,j}[n_i+n_j] &\leq \sum_{i\neq j, i,j\in[N]}\E_0[n_i+n_j] + \frac{T\epsilon}{2}\sum_{i\neq j, i,j\in[N]}\sqrt{\E_0[n_i+n_j]} \nonumber \\
    &\leq (N-1)\sum_{i=1}^N\E_0[n_i] + \frac{T\epsilon}{2}(N-1)\sum_{i=1}^N\sqrt{\E_0[n_i]} \nonumber \\
    &\leq (N-1)T + \frac{T\epsilon(N-1)}{2}\sqrt{NT}, \label{eqn:gap_i_j}
\end{align}
where the second inequality uses $\sqrt{a+b}\leq \sqrt{a}+\sqrt{b}$ and the third inequality is due to Cauchy-Schwarz inequality.

Applying \pref{eqn:gap_i_j} to \pref{eqn:reg_lower_bound}, we obtain that
\begin{align*}
    \E_{\calE}[\Reg] &\geq\epsilon T - \frac{2\epsilon}{N(N-1)}\cdot \left[(N-1)T+\frac{T\epsilon(N-1)}{2}\sqrt{NT}\right] \\
    &=\epsilon T - \frac{2\epsilon T}{N} - \epsilon^2N^{-\frac{1}{2}}T^{\frac{3}{2}} \\
    &\geq \frac{1}{3}\epsilon T - \epsilon^2N^{-\frac{1}{2}}T^{\frac{3}{2}} \tag{$N\geq 3$}.
\end{align*}
Picking $\epsilon=\frac{1}{4}\sqrt{\frac{N}{T}}$ leads to $\E_{\calE}[\Reg]\geq \frac{1}{48}\sqrt{NT}$. Therefore, there exists one environment among $\{\rho^{(i,j)}\}_{i\neq j, i,j\in[N]}$ such that $\E_{i,j}[\Reg]\geq \frac{1}{48}\sqrt{NT}$, which finishes the proof.
\end{proof}
\section{Omitted Proofs in~\pref{sec: TS}}\label{app: TS}
\begin{proof}[Proof of \pref{thm:TS}]
Let $i_t^*=\argmax_{i\in[N_t]}b_{t,i}\rho_{t,i}$, $j_t^*=\argsmax_{j\in[N_t]}b_{t,i}\rho_{t,i}$,
$i_{t,f}=\argmax_{i\in[N_t]}b_{t,i}f(x_t,i)$, and $j_{t,f}=\argsmax_{i\in[N_t]}b_{t,i}f(x_t,i)$.
 Also recall the notation $i_t=\argmax_{i\in[N_t]}b_{t,i}\wt{\rho}_{t,i}$ and $j_t=\argsmax_{j\in[N_t]}b_{t,i}\wt
{\rho}_{t,i}$. Then, we decompose the regret as follows:
\begin{align*}
    \E\left[\Reg\right] 
    &= \E\left[\sum_{t=1}^T\smax_{i\in[N_t]}\; b_{t,i}\rho_{t,i} - \sum_{t=1}^Tc_td_t\right] \\
    &=\E\left[\sum_{t=1}^T \rho_{t,j_t^*}b_{t,j_t^*} - \sum_{t=1}^T\frac{\wt{\rho}_{t,j_t}b_{t,j_t}}{\wt{\rho}_{t,i_t}}\rho_{t,i_t}\right] \\
    &= \E\left[\sum_{t=1}^T\E_{i_t, j_t}\left[\frac{\wt{\rho}_{t,j_t}b_{t,j_t}}{\wt{\rho}_{t,i_t}}\left(\wt{\rho}_{t,i_t} - \rho_{t,i_t}\right)\right]\right]  - \E\left[\sum_{t=1}^T\left(\wt{\rho}_{t,j_t}b_{t,j_t} - \rho_{t,j_t^*}b_{t,j_t^*}\right)\right] \\
    &= \E\left[\sum_{t=1}^T\E_{f\sim q_t}\left[\frac{f(x_t,j_{t,f})b_{t,j_{t,f}}}{f(x_t,i_{t,f})}\left(f(x_t,i_{t,f}) - f^*(x_t,i_{t,f})\right)\right]\right]  \\
    &\qquad - \E\left[\sum_{t=1}^T\underbrace{\left(f_t(x_t,j_t)b_{t,j_t} - f^*(x_t,j_t^*)b_{t,j_t^*}\right)}_{\triangleq\FG_t}\right] \\
    &\leq \E\left[\sum_{t=1}^T\E_{f\sim q_t}\left[|f(x_t,i_{t,f}) - f^*(x_t,i_{t,f})|\right]\right]  - \E\left[\sum_{t=1}^T\FG_t \right]
\end{align*}
where the inequality is because $\frac{f(x_t,j_{t,f})b_{t,j_{t,f}}}{f(x_t,i_{t,f})} \leq b_{t, i_{t,f}} \leq 1$.
Note that the ``Feel-Good'' term $\FG_t$ is the difference between what $f_t$ believes its payment is in its predicted world and that of the perfect predictor $f^*$.
We now analyze the other term for each time $t$, which can be written as:
\begin{align*}
    &\sum_{i\in[N_t]: p_{t,i} \neq 0}\E_{f\sim q_t}\left[\mathbbm{1}\{i_{t,f}=i\}| f(x_t,i) - f^*(x_t,i)|\right].
\end{align*}
where $p_{t,i}$ is the probability of $i_t$ being $i$.
Now consider a fixed $i$. We have for any $\mu > 0$,
\begin{align*}
    &\E_{f\sim q_t}\left[\mathbbm{1}\{i_{t,f}=i\}|f(x_t,i) -f^*(x_t,i)|\right] \\
    &\leq \E_{f\sim q_t}\left[\frac{\mathbbm{1}\{i_{t,f}=i\}}{4\mu p_{t,i}} +  \mu p_{t,i}\left(f(x_t,i) - f^*(x_t,i)\right)^2\right] \tag{AM-GM inequality}\\
    &= \frac{1}{4\mu} + \mu p_{t,i}\E_{f\sim q_t}\left[\left(f(x_t,i)- f^*(x_t, i)\right)^2\right], \tag{$p_{t,i} = \E_{f\sim q_t}[\mathbbm{1}\{i_{t,f}=i\}]$}
\end{align*}
which, after taking the summation over $i$, becomes
\[
\frac{N_t}{4\mu} + \mu\E_{i_t\sim p_t}\E_{f\sim q_t}\left[\left(f(x_t,i_t)- f^*(x_t, i_t)\right)^2\right]
\]
(note the important decoupling effect here: $f$ is not $f_t$).
With notation $\LS_t \triangleq (f(x_t,i_t)- f^*(x_t,i_t))^2$ (``Least Squares''), we have thus shown for any $\mu > 0$:
\begin{align}\label{eqn:regret_ts}
    \E[\Reg] \leq \frac{1}{4\mu}\sum_{t=1}^T N_t + \mu\E\left[\sum_{t=1}^T\E_{i_t\sim p_t}\E_{f\sim q_t}[\LS_t]\right] - \E\left[\sum_{t=1}^T\E_{f_t\sim q_t}[\FG_t]\right].
\end{align}

Next, using the fact that $c_t$ is a Bernoulli random variable with mean $f^*(x_t,i_t)$ and following a similar analysis of Lemma $4$ in~\citep{zhang2022feel}, we can show that 
\begin{align}\label{eqn:potential_diff}
    \frac{1}{16\eta}\E_{f\sim q_t}[\LS_t] -  \E_{f_t\sim q_t}[\FG_t]\leq -\frac{1}{\eta}\log\E_{c_t|x_t,i_t}\E_{f\sim q_t}\left[\exp\left(-\eta\left(\ellhat_{t,f}-\ellhat_{t,f^*}\right)\right)\right] + 4\eta,
\end{align}
with $\ellhat_{t,f}$ defined as the \optsq option shown in \pref{alg:EXP4}. For completeness, we include the proof of \pref{eqn:potential_diff} in \pref{lem:potential_diff}. 
Comparing \pref{eqn:regret_ts} and \pref{eqn:potential_diff} naturally suggests picking $\mu = \frac{1}{16\eta}$, so that
\begin{align}\label{eqn:regret_ts2}
\E[\Reg] \leq 4\eta\sum_{t=1}^T N_t + 4T\eta - \frac{1}{\eta}\sum_{t=1}^T \log\E_{c_t|x_t,i_t}\E_{f\sim q_t}\left[\exp\left(-\eta\left(\ellhat_{t,f}-\ellhat_{t,f^*}\right)\right)\right].
\end{align}

Note that the analysis so far holds for any $q_t$. To complete the proof, we now use the specific form of $q_t$ defined in \pref{alg:EXP4}: $q_{t,f}\propto\exp\left(-\eta\sum_{s=1}^{t-1}\ellhat_{s,f}\right)$. Using standard analysis of the multiplicative weight update, we show in \pref{lem:mwu} the following:
\begin{align}\label{eqn:mwu}
    -\E\left[\E_{i_t}\log 
    \E_{c_t|x_t,i_t}\;\E_{f\sim q_t} \; \left[\exp(-\eta(\ellhat_{t,f}-\ellhat_{t,f^*}))\right]\right]
  \leq Z_{t}-Z_{t-1},
\end{align}
where $Z_{t}=-\E\left[\log\E_{f\sim q_1}\exp\left(-\eta\sum_{\tau=1}^t\left(\ellhat_{t,f}-\ellhat_{t,f^*}\right)\right)\right]$.
 Combining this fact with \pref{eqn:regret_ts2}, we arrive at
\begin{align*}
    \E[\Reg]&\leq 4\eta\sum_{t=1}^TN_t + 4T\eta + \frac{1}{\eta}\sum_{t=1}^T (Z_t - Z_{t-1}) \\
    &\leq 4\eta\sum_{t=1}^TN_t+4T\eta+\frac{Z_T}{\eta}, \tag{telescoping and $Z_0 = 0$}
\end{align*}
which finishes the proof.
\end{proof}

\begin{lemma}\label{lem:potential_diff}
Suppose that $\eta\leq 1$. For any distribution $q_t$ over $\calF$ and $b_t\in[0,1]^N$, we have
\begin{align*}
    \frac{1}{16\eta}\E_{f\sim q_t}[\LS_t] - \E_{f_t\sim q_t}[\FG_t]  \leq -\frac{1}{\eta}\log\E_{c_t|x_t,i_t}\E_{f\sim q_t}\exp\left(-\eta\left(\ellhat_{t,f}-\ellhat_{t,f^*}\right)\right) + 4\eta,
\end{align*}
where $\LS_t$ and $\FG_t$ are defined in the proof of \pref{thm:TS}, and $\ellhat_{t,f}$ is defined in the \optsq option of \pref{eqn:ellhat}.
\end{lemma}
\begin{proof}
    Since $c_t\in\{0,1\}$ is a Bernoulli random variable with mean $f^*(x_t,i_t)$, $c_t$ satisfies the following sub-Gaussian random variable property: for any $\rho$,
    \begin{align*}
         \E_{c_t|x_t,i_t}\left[\exp\left(\rho\left( c_t - f^*(x_t,i_t) \right)\right)\right] \leq \exp\left(\frac{\rho^2}{8}\right).
    \end{align*}
  Let $\epsilon_t = c_t - f^*(x_t,i_t)$. For any $f$, setting $\rho = -\frac{1}{2}(f^*(x_t,i_t) - f(x_t,i_t))$, we thus have
  \begin{align}
    \E_{c_t|x_t,i_t} \exp\left( -\frac{1}{2}
    \epsilon_t \left(f^*(x_t,i_t) - f(x_t,i_t)\right) \right)
    \leq \exp\left(\frac{1}{32} \LS_t\right) . \label{eq:lem-exp-proof-1}
  \end{align}
  On the other hand, consider the following equalities:
  \begin{align*}
    &-\eta\left(\ellhat_{t,f}-\ellhat_{t,f^*}\right) \\
    &= -\frac{1}{4}(f(x_t,i_t)-c_t)^2 + \frac{1}{4}(f^*(x_t,i_t)-c_t)^2 + \eta\cdot \smax_jb_{t,j}f(x_t,j) - \eta\cdot\smax_jb_{t,j}f^*(x_t,j) \\
    & = 
    -\frac{1}{4}(\epsilon_t+f^*(x_t,i_t)-f(x_t,i_t))^2 + \frac{1}{4}
    \epsilon_t^2+ \eta\cdot \smax_jb_{t,j}f(x_t,j) - \eta\cdot\smax_jb_{t,j}f^*(x_t,j)\\
    &=
       -\frac{1}{2} \epsilon_t(f^*(x_t,i_t)-f(x_t,i_t)) - \frac{1}{4} \LS_t +
       \eta  \FG_t(f),
  \end{align*}
  where we define $\FG_t(f) = \smax_j(f(x_t,j)b_{t,j}) - \smax_j(f^*(x_t,j)b_{t,j})$ (so $\FG_t=\FG_t(f_t)$). Combining the above with \pref{eq:lem-exp-proof-1}, we obtain
\begin{align*}
    \E_{c_t |x_t , i_t}\left[\exp\left(-\eta\left(\ellhat_{t,f}-\ellhat_{t,f^*}\right)\right) \right]
    \leq \exp \left(- \frac{7}{32} \LS_t + \eta \FG_t(f) \right) .
\end{align*}
This further shows:
  \begin{align}
   & \log \E_{f\sim q_t}\left[\E_{c_t|x_t,i_t} \left[\exp\left(-\eta\left(\ellhat_{t,f}-\ellhat_{t,f^*}\right)\right)\right]\right]\nonumber\\
    &\leq \log \E_{f\sim q_t}  \left[\exp \left(- \frac{7}{32} \LS_t + \eta \FG_t(f) \right)\right]\nonumber
    \\
    &\leq \frac{1}{2} \log \E_{f\sim q_t}\left[ \exp \left( -\frac{7}{16} \LS_t\right)\right] +\frac{1}{2} \log\E_{f\sim q_t}\left[ \exp(2 \eta\FG_t(f) )\right], \label{eq:lem-exp-proof-2}
  \end{align}
  where the last inequality is due to Cauchy-Schwarz inequality. For the first term, using the facts $0\leq\LS_t\leq 1$ and $e^x\leq 1+x+\frac{x^2}{2}$ for $x\leq 0$, we know 
  \begin{align*}
      \E_{f\sim q_t}\left[\exp\left(-\frac{7}{16}\LS_t\right)\right]\leq 1-\frac{7}{16}\E_{f\sim q_t}[\LS_t] + \frac{49}{256}\E_{f\sim q_t}[\LS_t^2]\leq 1-\frac{1}{8}\E_{f\sim q_t}[\LS_t],
  \end{align*}
  where the last inequality is because $\LS_t^2 \leq \LS_t$. Further using $\log(1+x)\leq x$ gives
  \begin{align}\label{eqn:LS}
      \frac{1}{2}\log\E_{f\sim q_t}\left[\exp\left(-\frac{7}{16}\LS_t\right)\right] \leq \frac{1}{2}\log\left(1-\frac{1}{8}\E_{f\sim q_t}[\LS_t]\right)\leq -\frac{1}{16}\E_{f\sim q_t}[\LS_t].
  \end{align}
Moreover, since $\eta \leq 1 $ and $|\FG_t(f)| \leq 1$, using $e^x\leq 1+x+2x^2$ for $x\leq 2$, we have
\begin{align*}
    \frac{1}{2} \log
    \E_{f\sim q_t} \left[\exp(2\eta \FG_t(f))\right] &\leq \frac{1}{2}\log\left(1+
    2\eta \E_{f\sim q_t}[\FG_t(f)] + 2(2\eta)^2 \right) \notag\\
   & \leq  \eta \E_{f\sim q_t} [\FG_t(f)] + 4\eta^2 \tag{$\log(1+x) \leq x$}\\
   &=     \eta \E_{f_t\sim q_t} [\FG_t] + 4\eta^2   \tag{$f_t$ is drawn from $q_t$}
   .
\end{align*}
 Plugging the last bound and \pref{eqn:LS} into
 \pref{eq:lem-exp-proof-2} and rearranging finishes the proof.
\end{proof}

\begin{lemma}\label{lem:mwu}
\pref{alg:EXP4} guarantees that for each $t\in[T]$,
\begin{align*}
    -\E\left[\E_{i_t\sim p_t}\log 
    \E_{c_t|x_t,i_t}\;\E_{f\sim q_t} \; \left[\exp(-\eta(\ellhat_{t,f}-\ellhat_{t,f^*}))\right]\right]
  \leq Z_{t}-Z_{t-1},
\end{align*}
where $Z_t=-\E\left[\log\E_{f\sim q_1}\left[\exp\left(-\eta\sum_{\tau=1}^t\left(\ellhat_{t,f}-\ellhat_{t,f^*}\right)\right)\right]\right]$.
\end{lemma}
\begin{proof}
Let $W_{t,f} \triangleq \exp\left(-\eta\sum_{\tau=1}^t\left(\ellhat_{t,f}-\ellhat_{t,f^*}\right)\right)$.
According to \pref{alg:EXP4}, we know that
\begin{align*}
    q_{t,f} = \frac{\exp\left(-\eta\sum_{\tau=1}^{t-1}\ellhat_{\tau,f}\right)}{\int_{f'\in\calF}\exp\left(-\eta\sum_{\tau=1}^{t-1}\ellhat_{\tau,f'}\right)df'} = \frac{W_{t-1,f}}{\int_{f'\in\calF}W_{t-1,f'}df'}.
\end{align*}
Then, according to the definition of $Z_t$, we have
\begin{align*}
     & Z_{t-1} - Z_{t} \\
     & = \E\left[ \log 
    \frac{\int_{f\in\calF}W_{t,f}df}{\int_{f\in\calF} \; W_{t-1,f}df}\right]\\
  &= \E\left[\log 
    \frac{\int_{f\in\calF}  \;
     W_{t-1,f} \exp(-\eta(\ellhat_{t,f}-\ellhat_{t,f^*}))df}{\int_{f\in\calF} W_{t-1,f}df}\right]\\
    &=\E \left[\log 
    \E_{f\sim q_t} \left[\exp(-\eta(\ellhat_{t,f}-\ellhat_{t,f^*})\right]\right]  \\
  &\leq \E\left[\E_{i_t\sim p_t}\log\E_{c_t|x_t,i_t}\E_{f\sim q_t}       \left[\exp(-\eta(\ellhat_{t,f}-\ellhat_{t,f^*})\right]\right],
\end{align*}
where the last inequality is due to Jensen's inequality. Rearranging the terms finishes the proof.
\end{proof}
\begin{proof}[Proof of \pref{cor:TS_Lipschitz}]
    To show that $Z_T=\order(\alpha+d\log BT)$, we consider a small cube around the true parameter $\theta^*$: $\Omega_T=\{\theta: \|\theta-\theta^*\|_\infty\leq \frac{1}{T}\}$. Since $\calF$ is $\alpha$-Lipschitz with respect to $\|\cdot\|_\infty$, we know that $\|f_\theta-f_{\theta^*}\|_\infty\leq \frac{\alpha}{T}$. Therefore, for any $\theta\in \Omega_T$,
    \begin{align}\label{eqn:deltaell}
        &-\eta(\ellhat_{t,f_\theta}-\ellhat_{t,f_{\theta^*}})\nonumber\\ &= -\frac{1}{4}(f_\theta(x_t,i_t) - c_t)^2 + \frac{1}{4} (f_{\theta^*}(x_t,i_t) - c_t)^2\nonumber \\
        &\qquad + \eta\cdot\smax_jb_{t,j}f_{\theta}(x_t,j) - \eta\cdot\smax_jb_{t,j}f_{\theta^*}(x_t,j) \nonumber\\
        &\geq -\frac{1}{2}|f_\theta(x_t,i_t)-f_{\theta^*}(x_t,i_t)| + \eta\cdot\smax_jb_{t,j}f_{\theta}(x_t,j) - \eta\cdot\smax_jb_{t,j}f_{\theta^*}(x_t,j).
    \end{align}
    For the second and the third term, if $\argmax_ib_{t,i}f_{\theta}(x_t,i)=\argmax_ib_{t,i}f_{\theta^*}(x_t,i)$, let $j_{\theta^*}=\argsmax_jb_{t,j}f_{\theta^*}(x_t,j)$ and we know that 
    \begin{align*}
        &\smax_jb_{t,j}f_{\theta}(x_t,j) - \smax_jb_{t,j}f_{\theta^*}(x_t,j) \\
        &\geq  b_{t,j_{\theta^*}}f_{\theta}(x_t,j_{\theta^*}) - \smax_jb_{t,j}f_{\theta^*}(x_t,j) \\
        &\geq - b_{t,j_{\theta^*}}|f_{\theta}(x_t,j_{\theta^*}) - f_{\theta^*}(x_t,j_{\theta^*})|.
    \end{align*}
    Otherwise, let $i_{\theta^*}=\argmax_ib_{t,i}f_{\theta^*}(x_t,i)$ and we know that
    \begin{align*}
        &\smax_jb_{t,j}f_{\theta}(x_t,j) - \smax_jb_{t,j}f_{\theta^*}(x_t,j) \\
        &\geq \smax_jb_{t,j}f_{\theta}(x_t,j) -  b_{t,i_{\theta^*}}f_{\theta^*}(x_t,i_{\theta^*}) \\
        &\geq  b_{t,i_{\theta^*}}f_{\theta}(x_t,i_{\theta^*}) -  b_{t,i_{\theta^*}}f_{\theta^*}(x_t,i_{\theta^*}) \\
        &\geq - b_{t,i_{\theta^*}}|f_{\theta}(x_t,i_{\theta^*}) - f_{\theta^*}(x_t,i_{\theta^*})|.
    \end{align*}
    Combining the two cases above and plugging them into \pref{eqn:deltaell}, we obtain
    \begin{align*}
        -\eta(\ellhat_{t,f_\theta}-\ellhat_{t,f_{\theta^*}}) \geq -\frac{\alpha}{2T} - \frac{\alpha\eta}{T}.
    \end{align*}
    This means that
    \begin{align*}
        Z_{T}&=-\E\left[\log\E_{f\sim q_1}\exp\left(-\eta\sum_{t=1}^T\left(\ellhat_{t,f_\theta}-\ellhat_{t,f_{\theta^*}}\right)\right)\right] \\
        &\leq -\E\left[\log \frac{1}{(BT)^d}\inf_{\theta\in\Omega_T}\exp\left(-\eta\sum_{t=1}^T\left(\ellhat_{t,f_\theta}-\ellhat_{t,f_{\theta^*}}\right)\right)\right] \\
        &\leq d\log BT+\frac{\alpha}{2} + \alpha \eta = \order(\alpha + d\log BT),
    \end{align*}
    which finishes the proof.
\end{proof}

\section{Omitted Proofs in~\pref{sec: squarecb}}\label{app: squarecb}

\begin{proof}[Proof of \pref{prop:dec}]
By the definition of $\dec$, we have
\[
\E[\Reg] \leq \E\left[\gamma \sum_{t=1}^T (\rho_{t,i_t} - \wh{\rho}_{t,i_t})^2 + \sum_{t=1}^T  \dec(\wh{\rho}_t)\right].
\]
It remains to notice that the first term is bounded by $\RegSq$ under \pref{asm:regression_oracle}:
\begin{align*}
&\RegSq \geq \E\left[\sum_{t=1}^T \left(f_t(x_t, i_t) - c_t\right)^2- \sum_{t=1}^T \left(f^*(x_t, i_t) - c_t\right)^2\right] \\
& =\E\left[\sum_{t=1}^T \left(f_t(x_t, i_t) - f^*(x_t, i_t)\right)\left(f_t(x_t, i_t) + f^*(x_t, i_t) - 2c_t)\right)\right] \\
&= \E\left[\sum_{t=1}^T \left(f_t(x_t, i_t) - f^*(x_t, i_t)\right)^2 \right] \tag{the conditional expectation of $c_t$ is $f^*(x_t, i_t)$} \\
&= \E\left[ \sum_{t=1}^T (\rho_{t,i_t} - \wh{\rho}_{t,i_t})^2 \right].
\end{align*}
This finishes the proof.
\end{proof}

\begin{proof}[Proof of \pref{thm:squareCB}]    
Consider the $\epsilon$-greedy strategy described above \pref{thm:squareCB} and any $\rho, b \in [0,1]^N$.
We have
\begin{align}
&\E_{\wt{\rho}\in Q}\Bigg[ \left(\smax_{i\in[N]} b_i \rho_i - \rho_{i^*}\frac{\smax_{j \in [N]} b_j \wt{\rho}_j }{\wt{\rho}_{i^*}} \right) - \gamma\left(\rho_{i^*} - \wh{\rho}_{i^*} \right)^2\Bigg] \notag\\
&=  \left(\smax_{i\in[N]} b_i \rho_i - \E_{\wt{\rho}\in Q}\Bigg[\rho_{i^*}\frac{\smax_{j \in [N]} b_j \wt{\rho}_j }{\wt{\rho}_{i^*}}\Bigg] \right) - \gamma\sum_{i=1}^N p_i (\rho_i - \wh{\rho}_i)^2 \notag\\
&\leq \left(\smax_{i\in[N]} b_i \rho_i - (1-\epsilon)\rho_{\wh{i}^*}\frac{\smax_{j \in [N]} b_j \wh{\rho}_j }{\wh{\rho}_{\wh{i}^*}} \right) - \gamma \sum_{i=1}^N p_i (\rho_i - \wh{\rho}_i)^2 \notag \\
&\leq \left(\smax_{i\in[N]} b_i \rho_i - \rho_{\wh{i}^*}\frac{\smax_{j \in [N]} b_j \wh{\rho}_j }{\wh{\rho}_{\wh{i}^*}} \right) - \gamma \sum_{i=1}^N p_i (\rho_i - \wh{\rho}_i)^2 + \epsilon \label{eqn:objective}
\end{align}
where in the first step we introduce the notation $p_i$ which is the probability of $i^*$ being $i$,
in the second step we ignore the revenue obtained from exploration and use $\wh{i}^*$ to denote $\argmax_i b_i \wh{\rho}_i$,
and in the last step we use the fact $\frac{\smax_{j \in [N]} b_j \wh{\rho}_j }{\wh{\rho}_{\wh{i}^*}}  \leq 1$.
We now make the following two observations for any $i \in [N]$, both due to AM-GM inequality and the fact $p_i \geq \epsilon / N$:
\begin{equation}\label{eqn:ob1}
 \quad |b_i \rho_i - b_i \wh{\rho}_i| - \frac{\gamma}{2}p_i(\rho_i - \wh{\rho}_i)^2
\leq \frac{b_i^2}{2\gamma p_i} \leq \frac{N}{2\gamma \epsilon},
\end{equation}
\begin{equation}\label{eqn:ob2}
\begin{split}
\frac{\smax_{j \in [N]} b_j \wh{\rho}_j }{\wh{\rho}_{i}}|\rho_{\wh{i}^*} - \wh{\rho}_{\wh{i}^*}|
- \frac{\gamma}{2}p_{\wh{i}^*}(\rho_{\wh{i}^*}- \wh{\rho}_{\wh{i}^*})^2 &\leq \frac{1}{2\gamma p_{\wh{i}^*}}\left(\frac{\smax_{j \in [N]} b_j \wh{\rho}_j }{\wh{\rho}_{\wh{i}^*}}\right)^2 \\
&\leq \frac{b_{\wh{i}^*}^2}{2\gamma p_{\wh{i}^*}} \leq  \frac{N}{2\gamma \epsilon}.
\end{split}
\end{equation}
We are now ready to bound \pref{eqn:objective} by considering two cases: $\wh{i}^* \neq j^*$ or $\wh{i}^* = j^*$ where $j^* = \argsmax_i b_i \rho_i$.
In the first case ($\wh{i}^* \neq j^*$), we use \pref{eqn:ob1} with $i = j^*$ and \pref{eqn:ob2} to continue to bound \pref{eqn:objective} as
\begin{align*}
b_{j^*} \wh{\rho}_{j^*} - \smax_{j \in [N]} b_j \wh{\rho}_j  + \frac{N}{\gamma\epsilon} + \epsilon \leq \frac{N}{\gamma\epsilon} + \epsilon,
\end{align*}
where the last step is because $j^*$ is not $\wh{i}^*$ and thus $b_{j^*} \wh{\rho}_{j^*}$ is not $\max_j b_j \wh{\rho}_j$ and must be at most the second max.
In the second case ($\wh{i}^* = j^*$), we know $\wh{i}^* \neq \argmax_i b_i \rho_i \triangleq k$,
and thus we first bound $\smax_{i\in[N]} b_i \rho_i$ by $b_k \rho_k$ and then apply \pref{eqn:ob1} with $i = k$ and \pref{eqn:ob2} to obtain the following upper bound:
\[
b_{k} \wh{\rho}_{k} - \smax_{j \in [N]} b_j \wh{\rho}_j  + \frac{N}{\gamma\epsilon} + \epsilon \leq \frac{N}{\gamma\epsilon} + \epsilon,
\]
where the last step is again because $k$ is not $\wh{i}^*$.
Picking $\epsilon = \sqrt{N/\gamma}$, we have thus proven $\dec(\wh{\rho}) \leq 2\sqrt{N/\gamma}$.

The second statement of theorem is then a direct consequence.
Indeed, following the analysis of \pref{prop:dec}, we know that \pref{alg:squareCB.A} ensures
\[
\E[\Reg] \leq \gamma \RegSq + 2\sum_{t=1}^T \sqrt{\frac{N}{\gamma}}
\]
where $\gamma = N/\epsilon^2$. Plugging the value of $\epsilon$ finishes the proof.
\end{proof}

\end{document}